\documentclass{article}

\usepackage[final,main]{neurips_2026}

\usepackage[utf8]{inputenc} 
\usepackage[T1]{fontenc}    
\usepackage{hyperref}       
\usepackage{url}            
\usepackage{booktabs}       
\usepackage{amsfonts}       
\usepackage{nicefrac}       
\usepackage{microtype}      
\usepackage{xcolor}         
\usepackage{times}

\usepackage{graphicx}
\usepackage{color}
\usepackage{rotating}
\usepackage{tabularx}
\usepackage{pdflscape}
\usepackage{amsmath,amsthm,amssymb}
\usepackage{algorithm,algorithmic}
\usepackage{bbm,dsfont}
\usepackage{subcaption}
\usepackage{wrapfig}
\usepackage{cancel}
\usepackage{enumerate,cases}
\usepackage{thmtools,thm-restate}
\usepackage{mathtools}

\usepackage{multirow}
\usepackage{makecell}
\usepackage{setspace}
\usepackage{pifont}
\usepackage{array}
\usepackage{enumitem}
\usepackage{lineno}     
\usepackage{tcolorbox}
\tcbuselibrary{skins}
\usepackage{placeins} 

\allowdisplaybreaks

\usepackage{hyperref}		 
\hypersetup{
}
\usepackage[capitalize]{cleveref}
\crefformat{figure}{Fig.~#2#1#3}       
\crefformat{table}{Tab.~#2#1#3}        
\crefformat{equation}{Eq.~(#2#1#3)}    
\crefformat{section}{Sec.~#2#1#3}      
\crefformat{appendix}{App.~#2#1#3}     
\crefformat{algorithm}{Alg.~#2#1#3}    
\crefformat{definition}{Def.~#2#1#3}       

\usepackage{silence}
\newcommand{\algo}{{\textsc{CoFlow}}}

\newcommand{\para}[1]{\textbf{#1}~}

\renewcommand{\phi}{\varphi}
\renewcommand{\epsilon}{\varepsilon}

\newcommand{\el}{\end{flushleft}}
\newcommand{\bl}{\begin{flushleft}}

\newcommand{\squishlisttwo}{
 \begin{list}{$\bullet$}
  { \setlength{\itemsep}{1pt}
     \setlength{\parsep}{0pt}
    \setlength{\topsep}{0pt}
    \setlength{\partopsep}{0pt}
    \setlength{\leftmargin}{1em}
    \setlength{\labelwidth}{1.5em}
    \setlength{\labelsep}{0.5em} } 
}
\newcommand{\squishend}{\end{list}}

\theoremstyle{plain}

\title{
    Contextual Flow Matching: Adaptive Step Selection in Flow Models for Efficient Visual Generation
}

\author{
  	Divya Jyoti Bajpai \\
  	Department of IEOR \\
  	IIT Bombay, Mumbai,	India \\
  	\texttt{divyajyoti.bajpai@iitb.ac.in }
    \And
	Arun Verma \\
 	SMART Centre \\
    Singapore \\
\texttt{arun.verma@smart.mit.edu}
  	\And
  	Manjesh K. Hanawal \\
  	Department of IEOR \\
  	IIT Bombay, Mumbai,	India \\
  	\texttt{mhanawal@iitb.ac.in}
 }

\begin{document}

\maketitle

    \begin{abstract}

Flow Matching enables high-quality visual generation via continuous-time dynamics, but inference remains costly due to multiple sequential function evaluations. Existing acceleration methods reduce the number of function evaluations but often introduce additional training overhead, degrade quality, or fail to account for input-dependent variability.
We propose \algo{}, an inference-time method that adaptively selects the step counts each generation based on the prompt features. Our context-aware \algo{} is trained online with an unsupervised reward that balances inference efficiency and generation fidelity.
Our method is plug-and-play, requiring no retraining of the underlying generative model. It generalizes to image and video generation, achieving over $2.5\times$ speedup while preserving perceptual and semantic quality. We further provide a theoretical analysis establishing an $O(1/K)$ forward-Euler discretization error bound under standard regularity conditions.

    \end{abstract}



    \section{Introduction}
    \label{sec:introduction}

\emph{Flow Matching (FM)} \citep{lipman2022flow} has recently emerged as a powerful paradigm for visual generation \citep{bond2021deep}. 
In this framework, a time-dependent vector field governs the evolution of samples, and integrating this dynamics produces high-quality outputs along smooth trajectories. 
Compared to diffusion-based approaches \citep{croitoru2023diffusion}, FM can achieve \emph{comparable or superior perceptual fidelity} with fewer function evaluations (steps), making it an attractive alternative for both image and video generation.
Despite these advantages, \emph{inference efficiency remains a central challenge} \citep{kornilov2024optimal, yan2024perflow, davtyan2026faster}. 
Generating a single sample requires \emph{sequential numerical integration} of the learned dynamics, which involves multiple forward passes of the model over time. 
This inherent serial processing becomes increasingly expensive as model capacity scales and tasks demand higher spatial resolutions or longer temporal horizons, resulting in substantial inference latency in practical deployments.

A range of methods have been proposed to reduce inference latency in visual generations. These methods include \emph{distillation} \citep{kornilov2024optimal,yan2024perflow,luhman2021knowledge}, \emph{step truncation} \citep{lu2022dpm, song2020denoising}, and \emph{consistency-based training} \citep{yang2024consistency, zhang2025inverse, dao2025self}. 
While these approaches effectively reduce the number of function evaluations, they introduce non-trivial trade-offs, including additional training costs and reliance on teacher models or large-scale datasets. 
Such training overhead can offset the computational benefits gained at inference time.
Several training-free methods have also been proposed that rely on heuristic step-skipping \citep{liu2025timestep, bajpai2026flowcast, bajpai2026fastflow} based on intermediate signals. 
However, these training-free methods require accurate, expensive estimates of step importance, and any estimation errors can accumulate over time, irreversibly degrading generation quality.


The difficulty of a generative task and associated computational requirements can vary substantially. For example, a simple prompt ``A photo of a cat" requires less computation than a prompt like ``A photo of a cat traveling to the moon in a UFO" (see Fig. \ref{fig:motivation_example}) due to the need for fine-grained and prolonged trajectory refinement of the generative process. The complexity of prompts can then influence the number of steps. AdaDiff \cite{zhang2025adadiff} adapts to such input-dependent complexity by using a pre-trained policy that decides the number of steps required to complete the generative task. However, pre-training such a policy is expensive, and the trained policy can be sensitive to distribution shifts at inference time. Our goal in this work is to develop an input-adaptive online strategy that eliminates the need for pre-training.

\begin{figure}[!t]
    \centering
    \includegraphics[width = \textwidth]{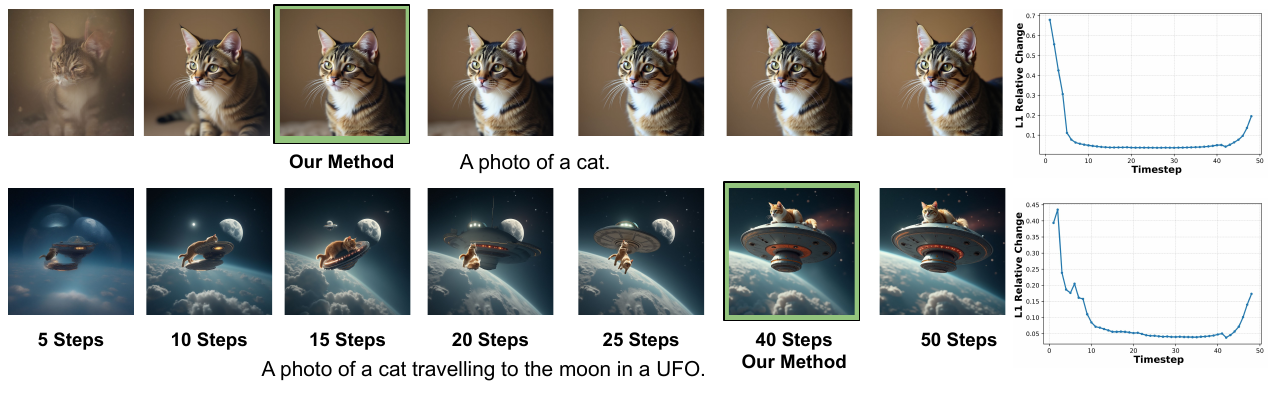}
    \caption{\textbf{Generation across integration step budgets.} Different inputs can exhibit different behavior as the integration budget changes. The rightmost plots illustrate relative velocity variation along the corresponding trajectories.}
    \label{fig:motivation_example}
    \vspace{-3mm}
\end{figure}

We introduce \algo{}: \underline{Co}ntextual \underline{Flow} Matching, an inference-time input-adaptive step selection method for Flow Models, thereby reducing inference cost.
\algo{} uses the \emph{ contextual bandit} \citep{zhou2020neural} framework to learn the optimal number of steps for each generation task in an online fashion. 
It treats prompt-specific features, like semantic embeddings, complexity cues, and prior observations from similar inputs as context, and uses a neural network to predict the number of steps required. Once the number of steps is decided, it executes the corresponding full generation trajectory. The neural network is trained on the fly by optimizing a reward that balances computational cost and trajectory refinement. 
The reward is motivated by the observation that complex prompts induce larger variations in the velocity field across timesteps (see Fig \ref{fig:motivation_example} and \ref{fig:reward_justification}).

\algo{} is fully task-agnostic, adapts to multiple existing solvers (see Appendix \ref{sec:solver_families}), and applies seamlessly to image generation, image editing, and video generation while maintaining high-fidelity outputs. We further provide theoretical grounding by deriving an upper bound on the discretization error induced by suboptimal step selection. Together, these components establish a principled adaptive inference framework that reduces unnecessary computation while preserving generation quality.

We summarize the key contributions of this paper as follows:
\vspace{-2mm}
\begin{itemize}
	\setlength\itemsep{0.0em}

    \item We propose a novel adaptive inference paradigm that predicts an input-dependent step budget using a neural contextual bandits algorithm, conditioned solely on prompt-specific features.
    
    
    \item We design an unsupervised reward mechanism based on trajectory dynamics, leveraging variations in the learned velocity field to balance computational cost and generation fidelity.
    
    \item We provide a theoretical analysis establishing an $O(1/K)$ discretization-error bound, highlighting the role of input-dependent flow dynamics in adaptive inference.

    \item Experimental results across image and video benchmarks show that our method achieves inference acceleration ($>2.5\times$) while maintaining quality and outperforming baselines.
\end{itemize}


    \section{Related Work}
    \label{sec:related_work}

Flow Matching (FM) \citep{lipman2022flow, dao2023flow, labs2025flux, deng2024fireflow} has emerged as a strong alternative to diffusion models \citep{croitoru2023diffusion, xing2024survey, yang2023diffusion, zhu2023conditional}.
FM achieves high-quality generation enabling applications such as image inversion \citep{deng2024fireflow}, editing \citep{wang2024taming}, and video synthesis \citep{kong2024hunyuanvideo, wan2025wan}. Recent advances further improves scalability and performance 
through multimodal extensions \citep{he2025flowtok,deng2025emerging}, large-scale flow models such as FLUX \citep{labs2025flux}, and hierarchical approaches such as pyramidal flow matching \citep{jin2024pyramidal} for video generation.

Despite these advances, iterative integration remains a fundamental bottleneck, limiting real-time and interactive deployment. 
Existing acceleration methods can be broadly categorized into retraining-based and inference-time approaches \citep{shaul2023bespoke}. 
Retraining-based methods, including distillation approaches such as InstaFlow \citep{liu2023instaflow}, LeDiFlow \citep{zwick2025lediflow}, and Diff2Flow \citep{schusterbauer2025diff2flow}, compress multi-step generation into one or a few steps by transferring knowledge from pretrained models. 
Other methods, such as PeRFlow \citep{yan2024perflow}, instead modify the training objective to produce straighter transport trajectories. 
Although effective, these approaches require additional training, large-scale data, or specialized objectives, resulting in substantial computational overhead.
Recent works further combine adversarial and consistency objectives to improve high-quality few-step generation \citep{yang2024consistency, zhang2025inverse, haber2025iterative}. 

Inference-time methods reduce computation by modifying the sampling procedure. Techniques such as TeaCache \citep{liu2025timestep} and FastFlow \citep{bajpai2026fastflow} skip redundant evaluations using trajectory-based signals, while FlowCast \citep{bajpai2026flowcast} adapts speculative decoding to flow matching models. While effective, these approaches typically rely on heuristic signals or alter trajectories designed for fixed discretization schedules, which can introduce inconsistencies and degrade generation quality.

Our setup is closely related to AdaDiff \citep{zhang2025adadiff}, but the solution approach differs in several fundamental aspects. The policy training in AdaDiff requires generating multiple images at various step budgets and assess their quality in each round of training, which require offline-data and introduces overheads. In contrast, our method is fully online and leverages intrinsic model dynamics (trajectory-level signals) as a proxy for generation complexity, thereby avoiding any additional generations and external evaluation.  
This yields a plug-and-play framework that preserves the consistency of the underlying discretization schdules of flow models resulting in efficient and high-quality visual outputs.


    \section{Methodology}
    \label{sec:methodology}

We begin with a brief review of Flow Matching (FM) and describe how our approach builds on it to enable adaptive inference.

\subsection{Flow Matching Overview}
Let $\pi_0$ and $\pi_1$ denote probability distributions over $\mathbb{R}^d$, corresponding to a simple source distribution and a complex target data distribution, respectively. Flow Matching formulates generation as learning a continuous transformation that maps samples from $\pi_0$ to $\pi_1$ via a time-dependent dynamical system. Specifically, a velocity field $v:\mathbb{R}^d \times [0,1] \rightarrow \mathbb{R}^d$ defines the evolution of a sample $x_t$ through the ordinary differential equation (ODE)
\begin{equation}
    \frac{dx_t}{dt} = v(x_t, t), \quad x_0 \sim \pi_0, \quad t \in [0,1].
\end{equation}
The velocity field, typically parameterized by a neural network, is trained such that integrating this dynamics transports the initial distribution $\pi_0$ to $\pi_1$ at $t=1$. This induces a flow map that traces a continuous trajectory for each sample, ultimately producing data-consistent outputs.

\textbf{Inference.} Since the learned dynamics generally do not admit closed-form solutions, generating requires numerical integration of the ODE. In practice, simple solvers such as the forward Euler method are commonly used due to their efficiency. The time interval $[0,1]$ is discretized into $K$ steps $\{t_0=0, t_1, \ldots, t_K=1\}$, and starting from $x_0 \sim \pi_0$, the sample is iteratively updated as
\begin{equation}
    x_{t_{k+1}} = x_{t_k} + \Delta t_k \cdot v(x_{t_k}, t_k), \quad \Delta t_k = t_{k+1} - t_k.
\end{equation}
This procedure approximates the continuous trajectory using a sequence of discrete updates, where the number of steps $K$ directly controls the trade-off between computational cost and generation quality. A finer discretization of the interval leads to lower Euler discretization error, denoted as $E(K)$ and vice versa. For a given $K$, the discretization points are fixed beforehand by the schedule of the solver.

A key observation motivating our approach is that different conditional inputs can induce trajectories with different levels of variation and, consequently, different sensitivities to temporal discretization. Thus, a fixed step budget may be unnecessarily conservative for some inputs while insufficient for others, motivating an adaptive mechanism that selects the step budget conditioned on the input.

\subsection{Algorithm Design}
To accelerate generation, our goal is to determine an input-dependent integration step budget during inference. Crucially, this decision should account for both the global data distribution and input-specific context, as samples from the same distribution can exhibit different trajectory dynamics. We cast adaptive step-budget selection as an online sequential decision-making problem using a Neural Contextual Bandit (NCB).

\begin{figure}[!ht]
    \centering
    \includegraphics[scale=0.619]{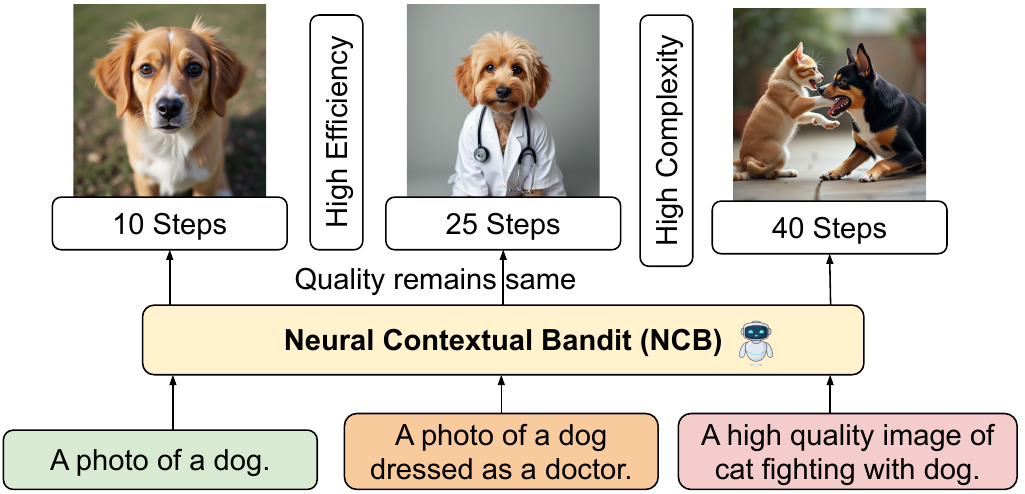}
    \caption{\textbf{Overview of our method.} \algo{} uses prompt context to adaptively select the integration step budget, allocating computation according to the learned surrogate reward while retaining the original flow trajectory and solver schedule.}
    \label{fig:placeholder}
\end{figure}

In a contextual bandit framework, an agent selects actions conditioned on input context to maximize cumulative reward while balancing exploration and exploitation. Let $\mathcal{A}$ denote the set of step budgets, where each action $\alpha \in \mathcal{A}$ corresponds to the number of discretization steps in $[0,1]$. For a given prompt $x$, let $\phi(x)$ denote the context features (e.g., semantic embeddings of prompt and complexity cues) enabling input-adaptive selection. A neural network models a policy $\pi(\alpha \mid \phi(x))$, mapping features to action preferences, capturing complex input–decision relationships. At each prompt $x$, the policy outputs the number of discretization steps $\alpha$ based on $\phi(x)$, and the full trajectory with $\alpha$ steps is executed. The policy is updated online using a reward below. 

\para{Reward design.}
Let $\{x_{t_k}^{(\alpha)}\}_{k=0}^{\alpha}$ denote the discretized trajectory obtained using step budget $\alpha\in\mathcal{A}$ for an input $x$, and let $v(x_{t_k}^{(\alpha)},t_k)$ denote the corresponding model velocity. We quantify its trajectory-wise variation using the average $\ell_1$ relative velocity change
\begin{equation}
\hat{v}(x,\alpha)=\frac{1}{\alpha-1}\sum_{k=1}^{\alpha-1}
\frac{\left\|v(x_{t_{k+1}}^{(\alpha)},t_{k+1})-v(x_{t_k}^{(\alpha)},t_k)\right\|_1}
{\max\left(\left\|v(x_{t_k}^{(\alpha)},t_k)\right\|_1,\varepsilon\right)},
\label{eq:velocity_change}
\end{equation}
where $\varepsilon>0$ ensures numerical stability. We use $\hat{v}(x,\alpha)$ as an intrinsic, unsupervised surrogate for sensitivity to temporal discretization: larger velocity variations suggest dynamics that may be more sensitive to coarse discretization. Importantly, $\hat{v}(x,\alpha)$ is a surrogate rather than a direct measure of discretization error or generation quality (see Appendix~\ref{sec:L1_as_proxy}).

For a given input $x$ and action $\alpha\in\mathcal{A}$, we define the reward as
\begin{equation}
r(x,\alpha)=\max\left(0,\,(1-\hat{v}(x,\alpha))-\mu\alpha\right),
\label{eq:reward}
\end{equation}
where $\mu>0$ controls the trade-off between trajectory variation and computational cost. The term $1-\hat{v}(x,\alpha)$ favors step budgets exhibiting smaller relative velocity variation, while $\mu\alpha$ penalizes additional computation. Since $\hat{v}(x,\alpha)$ is evaluated along the trajectory induced by $\alpha$, the reward provides an intrinsic surrogate objective for step selection without requiring an external quality signal.

For a given prompt $x$, we define the reward-optimal step budget as
\begin{equation}
\alpha^\star(x)\in
\arg\max_{\alpha\in\mathcal{A}}
\mathbb{E}\!\left[r(x,\alpha)\mid x\right].
\label{eq:optimal_budget}
\end{equation}
Here, ``reward-optimal'' refers specifically to optimality under the surrogate objective in Eq.~\ref{eq:reward}, rather than direct optimality with respect to perceptual or semantic generation quality. Consequently, prompts exhibiting larger trajectory variation can receive larger step budgets when the corresponding reduction in variation outweighs the additional computational cost. We further examine the relationship between this surrogate signal, prompt complexity, and adaptive step selection in Appendix~\ref{sec:L1_as_proxy}.

The goal is to learn a policy $\pi^*=\arg\max_{\pi}\mathbb{E}\!\left[\sum_{t=1}^{T}r(x_t,\alpha_t)\right]$, where the expectation accounts for stochasticity in the generated trajectories. This formulation captures the exploration--exploitation trade-off while enabling input-adaptive decisions. The learned policy uses the contextual representation and observed rewards to allocate computation according to the surrogate trade-off between trajectory variation and computational cost.

\textbf{Context features:} We define context features as a compact representation of the input prompt, capturing high-level semantic (embeddings), compositional (entity count), and structural (mean word length) information that guides adaptive step allocation. Intuitively, these properties may influence the variation of the conditional flow dynamics and therefore provide useful context for step selection. Appendix~\ref{sec:context_features} provides further details on the feature design.

\para{Policy optimization.}
We use a neural network  
$f_\theta(\phi(x), \alpha)$ to predict the reward for each pair of input context $\phi(x)$ 
and action $\alpha \in \mathcal{A}$ (step budget) denoted as
$
    f_\theta : (\phi(x), \alpha) \mapsto \mathbb{E}[r(x,\alpha)].
$
For the $t$-th prompt, we obtain the context $\phi(x_t)$, and decide the number of steps $\alpha_t$ by optimizing the sum of predicted rewards obtained by the neural network and the associated confidence term as follows: 
$
    \alpha_t = \arg\max_{\alpha \in \mathcal{A}} 
    \left[ f_\theta(\phi(x_t), \alpha) 
    + \beta_t \cdot \sigma_t(\phi(x_t), \alpha) \right],
$
where $\sigma_t(\phi(x_t), \alpha)$ is the confidence term and $\beta_t > 0$ controls the exploration strength.  The term $\sigma_t(\phi(x), \alpha)$ is computed using the gradient 
features of the neural network as in \cite{zhou2020neural}. Let $g_\theta(x, \alpha) = \nabla_\theta f_\theta(\phi(x), \alpha)$ denote the gradient of the neural network output with respect to its parameters. We maintain a covariance matrix
$
    V_t = \lambda I + \sum_{i=1}^{t-1} g_\theta(x_i, \alpha_i)\, g_\theta(x_i, \alpha_i)^\top,
$
where $\lambda > 0$ is a regularization parameter. The uncertainty is then defined as
$
    \sigma_t(\phi(x), \alpha) 
    = \sqrt{ g_\theta(x, \alpha)^\top V_t^{-1} g_\theta(x, \alpha) }.
$

Intuitively, $g_\theta(x,\alpha)$ represents the sensitivity of the predicted reward to the model parameters. If a context-action pair is similar to previously observed samples, its gradient direction is well-covered in $V_t$, resulting in a small uncertainty. Conversely, if it lies in a poorly explored region of the feature space, the corresponding direction has low coverage in $V_t$, leading to higher uncertainty. Thus, $\sigma_t(\phi(x), \alpha)$ quantifies how well the model has learned the reward for that context-action pair, naturally encouraging exploration of uncertain regions.

After executing the trajectory with $\alpha_t$ steps, reward $r_t = r(x_t, \alpha_t)$ is calculated as in Eq~\ref{eq:reward}. The model parameters are updated using all observed data and decision $\mathcal{D}_t = \{(\phi(x_i), \alpha_i, r_i)\}_{i=1}^t$. Specifically, the parameters are obtained by minimizing the empirical squared loss as follows:
\begin{equation}
    \theta_t = \arg\min_{\theta} 
    \sum_{i=1}^{t} 
    \left( f_\theta(\phi(x_i), \alpha_i) - r_i \right)^2.
\end{equation}
In practice, this optimization is performed periodically over a replay buffer of past interactions, allowing the model to progressively 
refine its reward estimates, leveraging both current context and historical data. We keep this buffer size as $P$.

\subsection{Algorithm: \algo{}}
We summarize our proposed method \algo{} in Algorithm~\ref{alg:ncb}. It takes as input a pretrained flow matching model $\mathcal{M}$, an action space $\mathcal{A}$ of step budgets, a neural reward model $f_{\theta}$, a regularization parameter $\lambda$, and an update interval $P$ controlling the frequency of training. We initialize the neural network parameter $\theta$, the covariance matrix $V_0 = \lambda I$ for uncertainty estimation, and a replay buffer $\mathcal{D}$. At each iteration $t$, a prompt $x_t$ is received and encoded into a context feature $\phi(x_t)$.

For $t \leq P$, actions are selected using an $\epsilon$-greedy strategy over $\mathcal{A}$ to encourage sufficient early exploration (line 5). After this warm-up phase, we switch to a contextual bandit regime in which actions $\alpha_t \in \mathcal{A}$ are chosen based on neural network-based reward estimates and their optimistic upper confidence bounds (line 7). We observe the trajectory generated by the flow-matching model with a solver schedule involving $\alpha_t$ steps. A reward is then computed from velocity variation across timesteps (Eq.~\ref{eq:reward}), capturing the surrogate trade-off between trajectory variation and computational cost. The tuple $(\phi(x_t),\alpha_t,r_t)$ is stored in $\mathcal{D}$.

\begin{algorithm}[t]
\caption{\algo{}: Contextual Flow Matching}
\label{alg:ncb}
\small
\textbf{Input:} Pretrained FM model $\mathcal{M}$, action set $\mathcal{A}$, neural reward model $f_\theta$, exploration schedule $\{\beta_t\}$, regularization $\lambda$, update interval $P$ \\
\textbf{Initialize:} Parameters $\theta$, covariance matrix $V_0=\lambda I$, replay buffer $\mathcal{D}\leftarrow\emptyset$
\begin{algorithmic}[1]
\FOR{each iteration $t=1,2,\dots$}
    \STATE Receive input $x_t$; extract context $\phi(x_t)$; compute $g_\theta(x_t,\alpha)=\nabla_\theta f_\theta(\phi(x_t),\alpha)$
    \STATE \textbf{Action selection:}
    \IF{$t\leq P$}
        \STATE Sample $\alpha_t$ using an $\epsilon$-greedy policy over $\mathcal{A}$
    \ELSE
        \STATE $\displaystyle \alpha_t=\arg\max_{\alpha\in\mathcal{A}}\left[f_\theta(\phi(x_t),\alpha)+\beta_t\sqrt{g_\theta(x_t,\alpha)^\top V_{t-1}^{-1}g_\theta(x_t,\alpha)}\right]$
    \ENDIF
    \STATE Initialize $x_{t_0}\sim\pi_0$ and generate the sample using $\mathcal{M}$ with $\alpha_t$ integration steps
    \STATE \textbf{Trajectory variation:} Compute $\hat v(x_t,\alpha_t)$ using Eq.~\ref{eq:velocity_change}
    \STATE \textbf{Reward:} $r_t=\max\!\left(0,(1-\hat v(x_t,\alpha_t))-\mu\alpha_t\right)$
    \STATE \textbf{Update:} $V_t=V_{t-1}+g_\theta(x_t,\alpha_t)g_\theta(x_t,\alpha_t)^\top$
    \STATE \textbf{Store:} $\mathcal{D}\leftarrow\mathcal{D}\cup\{(\phi(x_t),\alpha_t,r_t)\}$
    \STATE \textbf{If} {$t\bmod P=0$}
        \textbf{then} $\displaystyle \theta\leftarrow\arg\min_\theta\sum_{(\phi(x_i),\alpha_i,r_i)\in\mathcal{D}}\left(f_\theta(\phi(x_i),\alpha_i)-r_i\right)^2$
\ENDFOR
\end{algorithmic}
\end{algorithm}

\noindent\textbf{Theoretical Analysis.}
We characterize the dependence of numerical integration error on the step budget. Our analysis considers forward Euler under standard regularity conditions and shows that the error scales as $O(1/\alpha)$, with an input-dependent constant determined by variation of the conditional flow dynamics.

\begin{restatable}{thm}{main}
\label{thm:main}
For an input $x_i$, consider the conditional flow ODE
\[
\frac{dz_i(s)}{ds}=v_i(z_i(s),s),\qquad z_i(0)\sim\pi_0,\qquad s\in[0,1],
\]
where $v_i$ denotes the velocity field conditioned on $x_i$. Let $\hat z_{i,\alpha}^{(\alpha)}$ be the terminal state of its uniform $\alpha$-step forward-Euler approximation with $h=1/\alpha$, and define
$ E_i(\alpha):=\left\|z_i(1)-\hat z_{i,\alpha}^{(\alpha)}\right\|.$
Assume that $v_i(\cdot,s)$ is $L$-Lipschitz in its state argument and continuously differentiable along the exact trajectory, with
$
\left\|\partial_s v_i(z_i(s),s)+J_zv_i(z_i(s),s)v_i(z_i(s),s)\right\|\le M_i,\ s\in[0,1].
$
Then, for $L>0$,
$E_i(\alpha)\le\frac{M_i(e^L-1)}{2L\alpha}.$
Consequently, for any sequence $\{\alpha_i\}_{i=1}^{T}$,
\[
\sum_{i=1}^{T}E_i(\alpha_i)
\le\frac{e^L-1}{2L}\sum_{i=1}^{T}\frac{M_i}{\alpha_i}.
\]
For $L=0$, the bound reduces to $E_i(\alpha)\le M_i/(2\alpha)$.
\end{restatable}

The proof is provided in Appendix~\ref{sec:proof}. Theorem~\ref{thm:main} establishes the $O(1/\alpha)$ Euler error dependence, with the input-dependent constant $M_i$ bounding the rate of change of the velocity field along the exact trajectory. The statistic $\hat v(x,\alpha)$ in Eq.~\ref{eq:velocity_change} is an observable relative-variation surrogate motivated by this dependence; it is not assumed to directly estimate $M_i$ or the discretization error. Empirically, $\hat v$ exhibits a strong within-trajectory association with deviation from a high-step reference across integration budgets (Appendix~\ref{sec:L1_as_proxy}).

\noindent\textbf{Bandit Regret.}
Separately, the standard NeuralUCB analysis~\citep{zhou2020neural} provides a $\widetilde{O}(\sqrt{T})$ pseudo-regret guarantee under its stated assumptions. We use this result as a conditional characterization of the UCB component of \algo{} with respect to the surrogate reward in Eq.~\ref{eq:reward}. This result concerns learning the surrogate step-selection objective, whereas Theorem~\ref{thm:main} characterizes Euler discretization error as a function of the step budget.

\noindent\textbf{Computational Overhead.}
The contextual-bandit module introduces additional computation for uncertainty estimation and model updates. It uses a lightweight neural network ($d\sim30$K parameters), with the covariance inverse updated incrementally using rank-one Sherman--Morrison updates. The resulting forward, backward, and uncertainty computations are small relative to evaluations of the generative flow model, and their latency is included in all reported inference times. Detailed overhead measurements are provided in Appendix~\ref{sec:compute_for_contextual_bandit}.

    \section{Experiments}
    \label{sec:experiments}

We evaluate our method across three settings: text-to-image generation, image editing, and text-to-video synthesis. We first describe the datasets, followed by the baselines used for comparison.

\para{Datasets.}
For text-to-image generation, we use the \textbf{GenEval} benchmark \citep{ghosh2023geneval}, consisting of $553$ prompts designed to evaluate compositional reasoning. The prompts test capabilities such as \textit{object presence}, \textit{spatial relations}, \textit{attribute binding}, and \textit{counting}, providing a fine-grained measure of semantic alignment.

For image editing, we adopt \textbf{GEdit} \citep{liu2025step1x}, which contains $606$ real-world editing instructions spanning \textit{object manipulation}, \textit{appearance changes}, \textit{layout modifications}, and \textit{style transfer}. This enables evaluation of both localized and global transformations.
For video generation, we use \textbf{VBench} \citep{huang2024vbench} and construct a subset of $80$ prompts by uniformly sampling $5$ prompts from each of its $16$ evaluation dimensions. This ensures balanced coverage of factors such as \textit{motion consistency}, \textit{object persistence}, \textit{camera dynamics}, and \textit{scene composition}.

\para{Baselines.}
We compare against representative methods spanning full-generation, training-based acceleration, and inference-time approaches. \textbf{Full Generation} performs standard sampling over the complete trajectory without approximation, serving as an upper bound on fidelity, while \textbf{TeaCache} \cite{liu2025timestep} improves efficiency by caching intermediate representations across timesteps at the cost of minor approximation error.

\textbf{InstaFlow} \cite{liu2023instaflow} compresses multi-step generation into one or few steps via distillation, achieving large speedups but reduced visual and compositional quality. \textbf{PeRFlow} \cite{yan2024perflow} enforces piecewise-linear velocity trajectories through additional training for stable few-step generation. \textbf{FlowCast} \cite{bajpai2026flowcast} uses speculative decoding to predict future steps in parallel, reducing wall-clock time but increasing NFEs.

\textbf{FastFlow} \cite{bajpai2026fastflow} skips evaluations using a bandit-based smoothness heuristic, while \textbf{AdaDiff} \cite{zhang2025adadiff} formulates step selection as reinforcement learning guided by external quality metrics, introducing overhead and potential misalignment with semantic fidelity.

In contrast, \textbf{\algo{}} uses a neural contextual bandit to predict an input-adaptive step budget before sampling and executes the full trajectory accordingly, enabling context-aware computation while preserving discretization consistency; results include all samples encountered during exploration.

All baselines are evaluated using their official implementations and hyperparameters (see Appendix \ref{tab:hyperparams}). Table \ref{tab:main_results} reports results under released settings for all the baselines. We set $\mu = 1/\max_{\alpha\in\mathcal{A}}\alpha$ to normalize the computational-cost term relative to the available step budgets. The parameter $\mu$ controls the trade-off between trajectory variation and computational cost; its effect is examined in Figures~\ref{fig:res_editing} and~\ref{fig:res_video}. All the experiments were performed on a set of 2 NVIDIA A6000 GPUs.

\begin{table}[t]
\centering
\small
\setlength{\tabcolsep}{5pt}
\renewcommand{\arraystretch}{1.2}
\begin{tabular}{c|cccc|cccc}
\toprule
\textbf{Method} 
& \multicolumn{4}{c|}{\textbf{BAGEL}} 
& \multicolumn{4}{c}{\textbf{FLUX}} \\

\cmidrule(lr){2-5} \cmidrule(lr){6-9}
& Overall $\uparrow$ & IQA $\uparrow$ & NFE $\downarrow$ & Time $\downarrow$
& Overall $\uparrow$ & IQA $\uparrow$ & NFE $\downarrow$ & Time $\downarrow$ \\
\midrule

\multicolumn{9}{c}{\textit{Full Model}} \\
\midrule
Full-50 & \textbf{0.78} & \textbf{0.65} & 50 & 76.4 
        & \textbf{0.66} & \textbf{0.75} & 50 & 34.1 \\
Full-25 & 0.75 & 0.62 & 25 & 40.9 
        & 0.64 & 0.71 & 25 & 17.3 \\
Full-10 & 0.71 & 0.54 & 10 & 18.9 
        & 0.57 & 0.63 & 10 & 7.26 \\

\midrule
\multicolumn{9}{c}{\textit{Static Baselines}} \\
\midrule
InstaFlow & 0.33 & 0.49 & 1 & 2.1 
          & -- & -- & -- & -- \\
PerFlow   & 0.58 & 0.55 & 10 & 16.5 
          & -- & -- & -- & -- \\
TeaCache  & 0.74 & 0.59 & 29.4 & 45.8 
          & 0.63 & 0.72 & 28.4 & 20.8 \\

\midrule
\multicolumn{9}{c}{\textit{Adaptive Methods}} \\
\midrule
AdaDiff  & 0.72 & 0.58 & 28.1 & 43.6 
         & 0.59 & 0.67 & 25.8 & 18.2 \\
FlowCast & 0.76 & 0.60 & 21.5 & 34.2 
         & 0.65 & 0.72 & -- & 17.9 \\
FastFlow & 0.75 & 0.59 & 20.9 & 32.9 
         & 0.64 & 0.70 & 23.6 & 16.6 \\

\midrule
\multicolumn{9}{c}{\textit{Ours}} \\
\midrule
\textbf{\algo{}} & \textbf{0.77} & \textbf{0.63} & \textbf{17.7} & \textbf{27.5}
                  & \textbf{0.65} & \textbf{0.74} & \textbf{19.5} & \textbf{15.1} \\

\bottomrule
\end{tabular}
\caption{Comparison of methods on BAGEL and FLUX. We report generation quality (GenEval, IQA) along with efficiency metrics (number of function evaluations and inference time).}
\label{tab:main_results}
\vspace{-4mm}
\end{table}

\textbf{Models.} To demonstrate the versatility of our approach, we evaluate across multiple state-of-the-art models: for image generation, we use BAGEL \citep{deng2025emerging}, Flux-Kontext \citep{labs2025flux}, PeRFlow \citep{yan2024perflow} and StableDiffusion models \citep{podell2023sdxl}; for image editing, we use BAGEL and Flux-Kontext;  for video generation, we use the HunYuanVideo model \citep{kong2024hunyuanvideo}. 

\para{Training the neural contextual bandit.}
We model the expected reward using a two-layer MLP $f_\theta(\phi(x), \alpha)$ with ReLU activations and hidden dimension $64$ (adapting from NeuralUCB \cite{zhou2020neural}). The context $\phi(x)$ combines 8-dimensional handcrafted complexity features (see Appendix \ref{sec:context_features} and Table \ref{tab:prompt_features}) with projected semantic embeddings from a pretrained SentenceTransformer (all-MiniLM-L6-v2) \citep{wang2020minilm} (more details in Appendix \ref{sec:choice_of_minilm}). The Neural Bandit inputs the context and the action and predicts a reward for the given action.

The model is trained online by storing $(\phi(x_t), \alpha_t, r_t)$ in a replay buffer $\mathcal{D}$ and performing updates every $P$ iterations via MSE minimization over $\mathcal{D}$, using Adam \citep{kingma2014adam} with learning rate $5 \times 10^{-3}$ for $20$ epochs. Exploration is handled via a cold-start phase (each arm pulled once), followed by a $\epsilon$-greedy strategy, and subsequently a NeuralUCB-based action selection rule.

\begin{figure}[!ht]
    \centering
    
    \begin{subfigure}{0.97\textwidth}
        \centering
        \includegraphics[width=\linewidth]{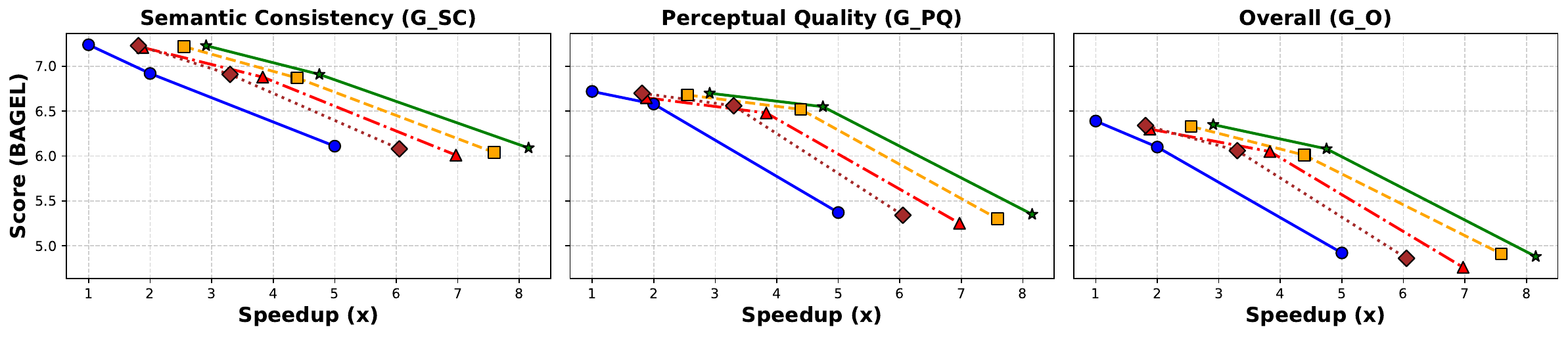}
    \end{subfigure}
    
    \vspace{0.5em}

    \begin{subfigure}{0.97\textwidth}
        \centering
        \includegraphics[width=\linewidth]{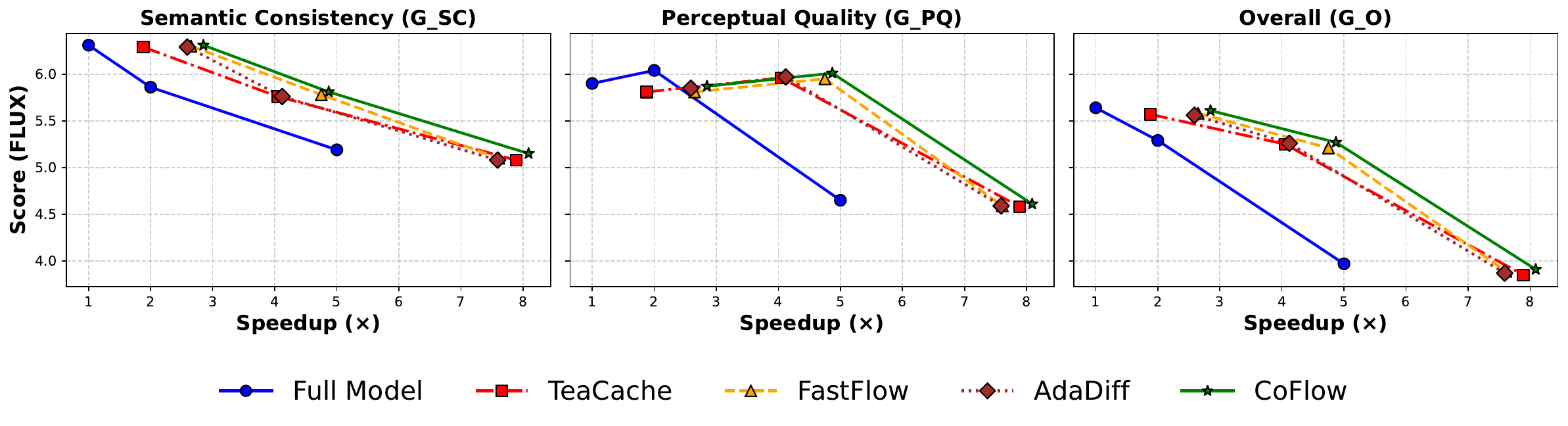}
    \end{subfigure}
    
    \vspace{-0.3cm}

    \caption{Comparison of edit quality across two models: BAGEL and FLUX. Each subfigure reports semantic consistency (G\_SC), perceptual quality (G\_PQ), and overall score (G\_O) versus speedup.}
    \label{fig:res_editing}
    \vspace{-3mm}
\end{figure}

\begin{figure}[!ht]
    \centering
    \includegraphics[scale = 0.29]{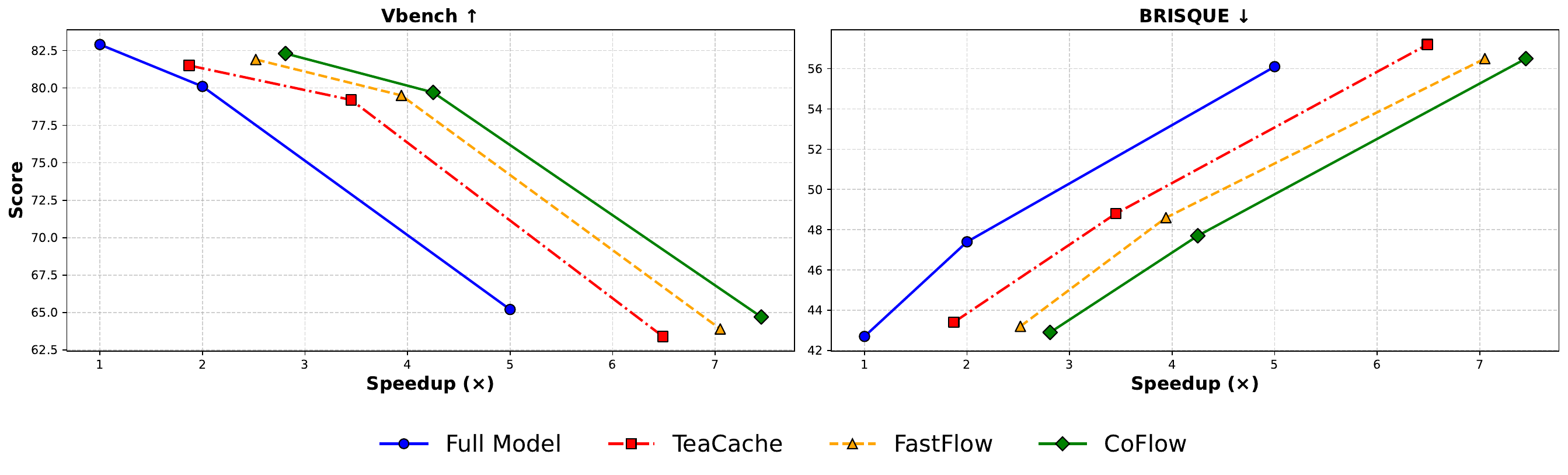}
    \caption{Comparison of video generation score and quality where we report the VBench score and the BRISQUE metric that assesses the quality of video.}
    \label{fig:res_video}
    \vspace{-4mm}
\end{figure}

\subsection{Results}

\para{Image generation.}
Table~\ref{tab:main_results} reports results on GenEval for BAGEL and FLUX, with additional evaluation on the optimized PeRFlow model in Table~\ref{tab:res_perflow}. GenEval measures fine-grained semantic alignment across compositional dimensions, while CLIP-IQA evaluates perceptual quality. On BAGEL, \algo{} achieves an overall score of $0.77$ with $17.7$ NFE, compared with $0.78$ at $50$ NFE for full inference; on FLUX, it achieves $0.65$ with $19.5$ NFE, compared with $0.66$ at $50$ NFE. These results demonstrate substantial computation reduction with a small change in aggregate generation quality. Category-wise results are provided in Appendix~\ref{sec:expanded_res}.

\para{Image editing.}
Figure~\ref{fig:res_editing} presents results on GEdit using BAGEL, FLUX, and Step-1X-Edit. We report G\_SC (semantic consistency), G\_PQ (perceptual quality), and G\_O (overall edit quality), evaluated using GPT-4.1~\citep{achiam2023gpt} as an automatic judge. Across the evaluated operating points, \algo{} reduces the integration budget while maintaining competitive semantic consistency and perceptual quality, demonstrating that adaptive step selection also extends to instruction-based image editing.

\para{Video generation.}
Figure~\ref{fig:res_video} reports results on VBench, which evaluates motion dynamics, temporal consistency, and scene composition; we additionally report BRISQUE for frame-level perceptual quality. \algo{} reduces the integration budget while maintaining competitive video-generation performance, indicating that input-adaptive step selection extends beyond image generation to temporally structured outputs.

\para{Summary.}
Across image generation, editing, and video synthesis, \algo{} consistently reduces inference computation while maintaining competitive output quality. The method selects an input-dependent integration budget and then executes the corresponding solver trajectory without modifying the pretrained generative model. Ablations further show that conditioning the policy on input context improves the observed quality--efficiency trade-off relative to a non-contextual variant (Appendix~\ref{sec:context_matters}). A finer-grained analysis further examines the contributions of MiniLM embeddings and handcrafted prompt features. Additional online ablations show that the policy stabilizes over the interaction sequence and is robust across prompt orderings, while uncertainty-aware UCB exploration improves both the surrogate objective and computational efficiency relative to an otherwise identical greedy policy (Appendix~\ref{sec:context_matters}). Finally, \algo{} can be combined with existing inference-time acceleration methods such as FastFlow and FlowCast for additional runtime reductions (Appendix~\ref{sec:complementary_method}).

\para{Limitations.}\label{sec:limitations}
While \algo{} is lightweight and fully online, it requires a brief exploration phase to learn the contextual policy; however, note that all reported results already include exploration samples. Its effectiveness also depends on the expressiveness of the prompt features, which are intentionally designed to be simple and efficient. More expressive features could further improve adaptation quality, but at the cost of additional computational overhead. In practice, both the exploration phase and lightweight feature design stabilize quickly and have minimal impact on overall performance.

    \section{Conclusion}
    \label{sec:conclusion}

We introduced \algo{}, an online framework for input-adaptive step selection in flow-based generative models. \algo{} formulates step-budget selection as a neural contextual bandit, using lightweight prompt representations and an intrinsic trajectory-variation surrogate to adapt computation without retraining the underlying generative model. Across image generation, image editing, and video synthesis, \algo{} substantially reduces inference computation while maintaining competitive semantic and perceptual quality across multiple backbones. We further provide a forward-Euler analysis establishing an $O(1/K)$ discretization-error dependence under standard regularity conditions, together with empirical evidence that the proposed velocity-variation statistic tracks within-trajectory sensitivity to temporal discretization. Our results demonstrate that online contextual step selection provides a practical mechanism for adapting inference computation across inputs while remaining compatible with pretrained flow models and existing inference-time acceleration techniques.





    
    \clearpage
    \bibliographystyle{unsrt}
    \bibliography{references}
    
    \clearpage
    \appendix
\section{Appendix}

\subsection{Proof of Theorem~\ref{thm:main}}
\label{sec:proof}

\begin{proof}
Fix an input $x_i$ and a step budget $\alpha$. Let $t_k=kh$, where $h=1/\alpha$ and $k=0,\ldots,\alpha$.
Let $z_i(s)$ denote the exact solution and $\hat z_{i,k}^{(\alpha)}$ the Euler iterate,
\[
\hat z_{i,k+1}^{(\alpha)}=\hat z_{i,k}^{(\alpha)}+h\,v_i(\hat z_{i,k}^{(\alpha)},t_k).
\]
Define the error
\[
e_k:=z_i(t_k)-\hat z_{i,k}^{(\alpha)}.
\]
Since both trajectories start from the same initial state, $e_0=0$.

By the fundamental theorem of calculus,
\[
z_i(t_{k+1})=z_i(t_k)+h\,v_i(z_i(t_k),t_k)+\tau_k,
\]
where the local truncation error is
\[
\tau_k:=\int_{t_k}^{t_{k+1}}\left[v_i(z_i(s),s)-v_i(z_i(t_k),t_k)\right]ds.
\]
Along the exact trajectory, the chain rule gives
\[
\frac{d}{ds}v_i(z_i(s),s)=\partial_s v_i(z_i(s),s)+J_zv_i(z_i(s),s)v_i(z_i(s),s).
\]
By the assumption of Theorem~\ref{thm:main},
\[
\left\|\frac{d}{ds}v_i(z_i(s),s)\right\|\le M_i.
\]
Hence,
\[
\left\|v_i(z_i(s),s)-v_i(z_i(t_k),t_k)\right\|\le M_i(s-t_k),
\]
and therefore
\[
\|\tau_k\|\le\int_{t_k}^{t_{k+1}}M_i(s-t_k)\,ds=\frac{M_i h^2}{2}.
\]

Subtracting the Euler update from the exact update yields
\[
e_{k+1}=e_k+h\left[v_i(z_i(t_k),t_k)-v_i(\hat z_{i,k}^{(\alpha)},t_k)\right]+\tau_k.
\]
Using the $L$-Lipschitz property of $v_i$,
\[
\|e_{k+1}\|\le(1+Lh)\|e_k\|+\frac{M_i h^2}{2}.
\]
Since $e_0=0$, iterating the recursion gives
\[
\|e_\alpha\|\le\frac{M_i h^2}{2}\sum_{k=0}^{\alpha-1}(1+Lh)^k.
\]
For $L>0$,
\[
\sum_{k=0}^{\alpha-1}(1+Lh)^k=\frac{(1+Lh)^\alpha-1}{Lh}.
\]
Therefore,
\[
\|e_\alpha\|\le\frac{M_i h}{2L}\left[(1+Lh)^\alpha-1\right].
\]
Since $h=1/\alpha$ and $\left(1+\frac{L}{\alpha}\right)^\alpha\le e^L$, we obtain
\[
E_i(\alpha)=\|e_\alpha\|\le\frac{M_i(e^L-1)}{2L\alpha}.
\]

For $L=0$, the recursion reduces to
\[
\|e_{k+1}\|\le\|e_k\|+\frac{M_i h^2}{2},
\]
and hence
\[
E_i(\alpha)\le\frac{\alpha M_i h^2}{2}=\frac{M_i}{2\alpha}.
\]

Finally, summing over any sequence of step budgets $\{\alpha_i\}_{i=1}^{T}$ gives, for $L>0$,
\[
\sum_{i=1}^{T}E_i(\alpha_i)\le\frac{e^L-1}{2L}\sum_{i=1}^{T}\frac{M_i}{\alpha_i}.
\]
This completes the proof.
\end{proof}

\paragraph{Discussion of the assumptions.}
The result requires standard regularity conditions for forward Euler: Lipschitz continuity of the velocity field in the state variable and a uniform bound $M_i$ on the derivative of the velocity field along the exact trajectory.
The constant $M_i$ captures input-dependent variation in the conditional flow dynamics.
The theorem characterizes numerical integration error independently of the surrogate reward used by \algo{} for adaptive step selection.

\begin{figure}
    \centering
    \includegraphics[width=1.0\linewidth]{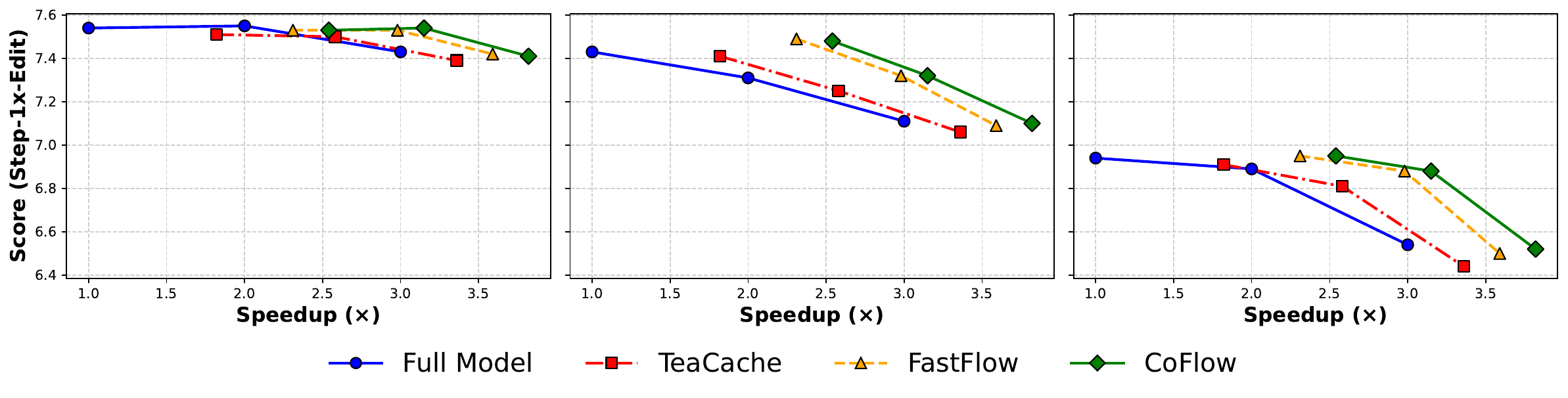}
    \caption{Comparison of edit quality on Step-1X-Edit model. Each subfigure reports semantic consistency (G\_SC), perceptual quality (G\_PQ), and overall score (G\_O) versus speedup.}
    \label{fig:res_edit_step_edit}
\end{figure}

\begin{figure}
    \centering
    \includegraphics[width=1.0\linewidth]{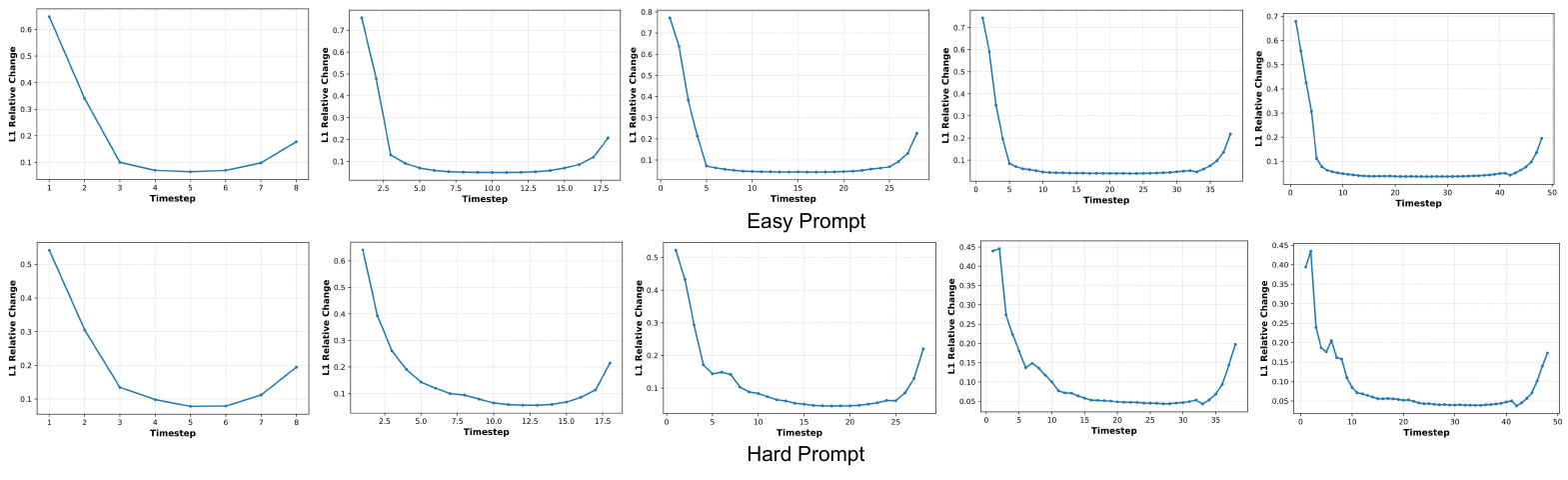}
    \caption{\textbf{L1-relative velocity variation across integration trajectories.} Velocity-change profiles vary across prompts and integration budgets, illustrating the trajectory-level signal used by \algo{}ƒ. Quantitative analysis against a 100-step reference is provided in Appendix~\ref{sec:L1_as_proxy}.}
    \label{fig:reward_justification}
\end{figure}

\subsection{Expanded Results on the GenEval Dataset}
\label{sec:expanded_res}

Tables~\ref{tab:res_expanded_BAGEL} and~\ref{tab:res_expanded_flux} provide a category-wise breakdown on GenEval across Single Object (SO), Two Object (TO), Counting (CO), Color (CL), Color Attribute (CA), and Position (PO). This analysis complements the aggregate results by examining semantic alignment and compositional performance across different prompt categories.

Across both backbones, \algo{} maintains competitive performance across the GenEval categories while using substantially fewer function evaluations than full-step inference. The category-wise results further show that the efficiency gains are not restricted to a single type of prompt, with \algo{} retaining comparable performance across object composition, counting, attribute binding, and spatial-relation tasks. These results provide a more detailed view of the quality--efficiency trade-off observed in the aggregate evaluation.

\subsection{Computational Overhead of the Contextual Bandit}\label{sec:compute_for_contextual_bandit}
\begin{wraptable}{r}{0.31\textwidth}
\vspace{-10pt}
\centering
\small
\setlength{\tabcolsep}{5pt}
\renewcommand{\arraystretch}{1.2}
\begin{tabular}{l c}
\toprule
\textbf{Method} & \textbf{Time (s)} \\
\midrule

\multicolumn{2}{c}{\textit{Full Model}} \\
\midrule
Full-50 & 75.6 \\
Full-25 & 40.1 \\

\midrule
\multicolumn{2}{c}{\textit{Ours}} \\
\midrule
NCB-50 & \textbf{76.4} \\
NCB-25 & \textbf{40.9} \\

\bottomrule
\end{tabular}
\vspace{-5pt}
\caption{Inference time comparison of \algo{}.}
\label{tab:time_comparison}
\vspace{-10pt}
\end{wraptable}
Table~\ref{tab:time_comparison} reports the inference time for standard generation using fixed step budgets (50 and 25 steps) and compares it with our method under the same fixed budgets (NCB-50 and NCB-25). This controlled setup isolates the overhead introduced by the contextual bandit, independent of any step reduction.

The results show that the additional latency incurred by our method is negligible ($<1$ second), demonstrating that the decision-making module introduces minimal runtime cost. Furthermore, training the neural bandit is computationally lightweight: even with a large number of interaction samples, the full training process completes within a few seconds ($\sim$3-4 seconds in our setup). We include this time while calculating the average inference time reported in our result Tables.

This efficiency stems from the simplicity of the underlying model, which consists of a small multilayer perceptron with $\approx30k$ parameters and low-dimensional inputs. Consequently, our approach achieves adaptive inference with virtually no computational overhead, making it practical for real-time deployment.

\subsection{Complementarity of Our Approach}
\label{sec:complementary_method}

Table~\ref{tab:complementary_results} evaluates \algo{} in combination with existing inference-time acceleration methods. While these methods reduce computation within a given sampling procedure, \algo{} operates at the level of selecting an input-dependent integration budget. The two mechanisms can therefore be applied jointly.

For each combination, \algo{} first selects the integration budget, after which the corresponding acceleration method is applied during generation. As shown in Table~\ref{tab:complementary_results}, combining adaptive step selection with TeaCache, FastFlow, and FlowCast further reduces runtime relative to using each acceleration method alone, although the magnitude of the gain and its effect on generation quality vary across methods.

These results demonstrate that adaptive step-budget selection can be combined with existing inference-time acceleration techniques, providing an additional mechanism for reducing computation without requiring modifications to the underlying generative model.

\begin{figure}
    \centering
    \includegraphics[scale=0.5]{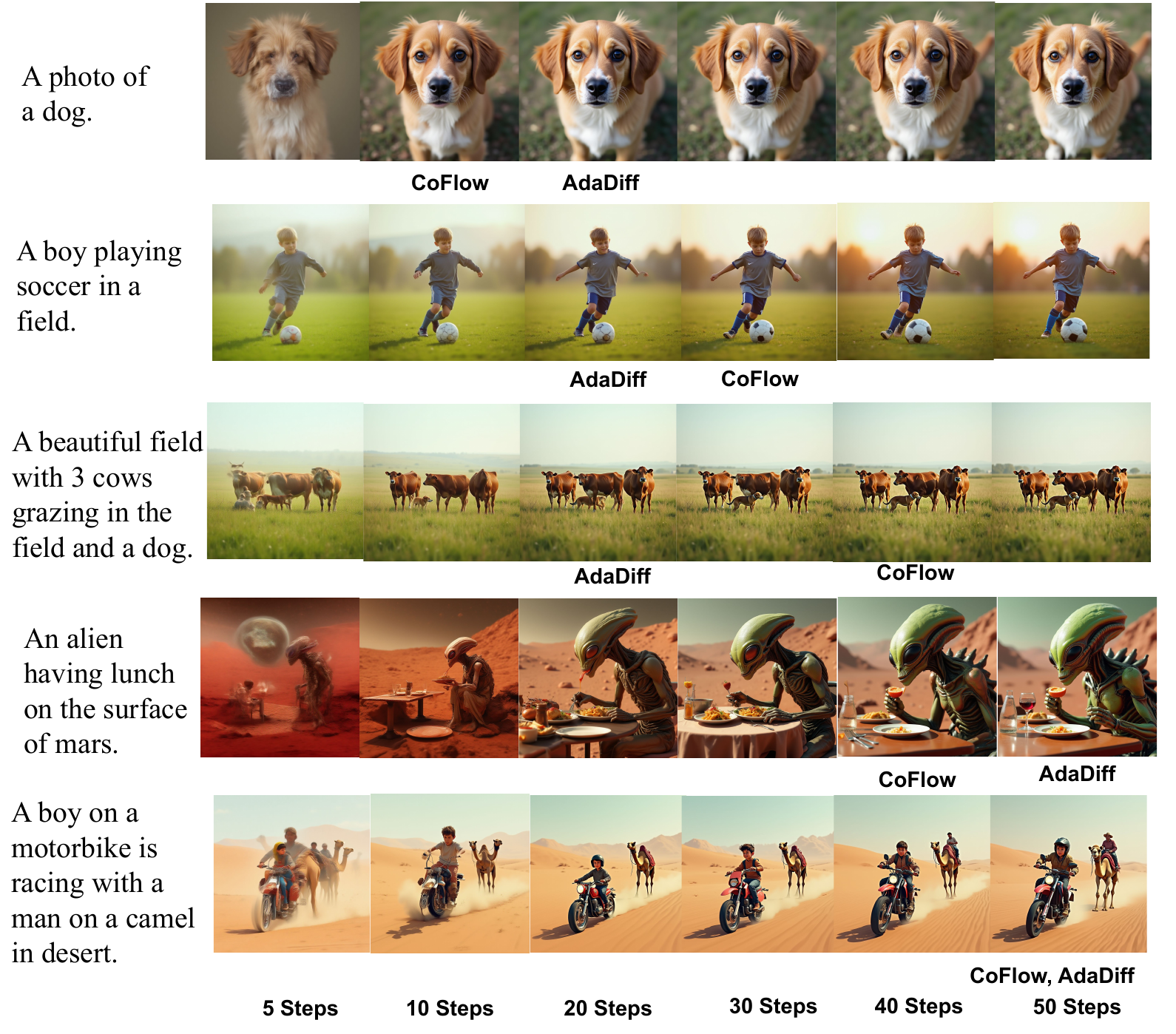}
    \caption{Qualitative examples showcasing the choices of AdaDiff and \algo{}.}
    \label{fig:qualitative_res}
\end{figure}

\subsection{Qualitative Analysis}

Figure~\ref{fig:qualitative_res} presents qualitative comparisons between \algo{} and AdaDiff together with their selected step budgets. AdaDiff uses a policy learned offline, whereas \algo{} updates its step-selection policy online from observed interactions.

The examples illustrate differences in step-allocation behavior between the two approaches. AdaDiff frequently selects either a relatively small budget or the maximum budget, while \algo{} also assigns intermediate step budgets across inputs. This behavior is consistent with the contextual online formulation of \algo{}, which allows its decisions to depend on the current prompt representation and accumulated observations.

Across the examples, \algo{} reduces the integration budget while retaining visually comparable outputs. These qualitative results complement the quantitative evaluations by illustrating how adaptive step selection can produce different computational allocations across individual inputs.

\begin{table}[t]
\centering
\small
\setlength{\tabcolsep}{4pt}
\renewcommand{\arraystretch}{1.2}

\begin{tabular}{l|cccccc|ccc}
\toprule
\textbf{Method} 
& \multicolumn{6}{c|}{\textbf{Compositional Dimensions}} 
& \multicolumn{3}{c}{\textbf{Overall Metrics}} \\

\cmidrule(lr){2-7} \cmidrule(lr){8-10}

& \textbf{SO} & \textbf{TO} & \textbf{CO} & \textbf{CL} & \textbf{CA} & \textbf{PO}
& \textbf{Overall} & \textbf{CLIPIQA} & \textbf{NFE} \\

\midrule
\multicolumn{10}{c}{\textit{Full Model}} \\
\midrule
Full-50 & 0.99 & 0.90 & 0.81 & 0.85 & 0.59 & 0.54 & 0.78 & 0.65 & 50 \\
Full-25 & 0.99 & 0.90 & 0.76 & 0.81 & 0.60 & 0.49 & 0.75 & 0.62 & 25 \\
Full-10 & 0.99 & 0.85 & 0.67 & 0.78 & 0.56 & 0.46 & 0.71 & 0.54 & 10 \\

\midrule
\multicolumn{10}{c}{\textit{Static Baselines}} \\
\midrule
InstaFlow & 0.85 & 0.19 & 0.20 & 0.65 & 0.04 & 0.02 & 0.33 & 0.49 & 1 \\
PeRFlow   & 0.99 & 0.79 & 0.43 & 0.85 & 0.25 & 0.15 & 0.58 & 0.55 & 10 \\
TeaCache  & 0.99 & 0.87 & 0.76 & 0.81 & 0.58 & 0.47 & 0.74 & 0.59 & 29.4 \\

\midrule
\multicolumn{10}{c}{\textit{Adaptive Methods}} \\
\midrule
AdaDiff  & 0.99 & 0.85 & 0.76 & 0.78 & 0.53 & 0.46 & 0.72 & 0.58 & 28.1 \\
FlowCast & 0.99 & 0.90 & 0.77 & 0.84 & 0.62 & 0.50 & 0.76 & 0.60 & 21.5 \\
FastFlow & 0.99 & 0.89 & 0.76 & 0.82 & 0.60 & 0.49 & 0.75 & 0.59 & 20.9 \\

\midrule
\multicolumn{10}{c}{\textit{Ours}} \\
\midrule
\textbf{Ours-UCB} 
& \textbf{0.99} & \textbf{0.90} & \textbf{0.78} & \textbf{0.84} & \textbf{0.61} & \textbf{0.51}
& \textbf{0.77} & \textbf{0.63} & \textbf{17.7} \\

\bottomrule
\end{tabular}

\caption{Comparison across various Compositional Dimensions and overall performance for the BAGEL model. Our method achieves strong performance while significantly reducing the number of function evaluations (NFE).}
\label{tab:res_expanded_BAGEL}
\end{table}

\begin{table}[t]
\centering
\small
\setlength{\tabcolsep}{4pt}
\renewcommand{\arraystretch}{1.2}

\begin{tabular}{l|cccccc|ccc}
\toprule
\textbf{Method} 
& \multicolumn{6}{c|}{\textbf{Compositional Dimensions}} 
& \multicolumn{3}{c}{\textbf{Overall Metrics}} \\

\cmidrule(lr){2-7} \cmidrule(lr){8-10}

& \textbf{SO} & \textbf{TO} & \textbf{CO} & \textbf{CL} & \textbf{CA} & \textbf{PO}
& \textbf{Overall} & \textbf{CLIPIQA} & \textbf{NFE} \\

\midrule
\multicolumn{10}{c}{\textit{Full Model}} \\
\midrule
Full-50 & 0.97 & 0.79 & 0.75 & 0.78 & 0.48 & 0.20 & 0.66 & 0.75 & 34.1 \\
Full-25 & 0.97 & 0.78 & 0.71 & 0.75 & 0.45 & 0.18 & 0.64 & 0.71 & 17.3 \\
Full-10 & 0.97 & 0.65 & 0.58 & 0.67 & 0.41 & 0.15 & 0.57 & 0.63 & 7.26 \\

TeaCache & 0.97 & 0.76 & 0.73 & 0.74 & 0.44 & 0.18 & 0.63 & 0.72 & 20.8 \\

\midrule
\multicolumn{10}{c}{\textit{Ours}} \\
\midrule
\textbf{Ours-UCB} 
& \textbf{0.97} & \textbf{0.79} & 0.73 & \textbf{0.76} & \textbf{0.47} & \textbf{0.19}
& \textbf{0.65} & \textbf{0.74} & \textbf{15.1} \\

\bottomrule
\end{tabular}

\caption{Comparison of the Compositional Dimensions for the FLUX model. Our method maintains competitive compositional performance while reducing computation (NFE).}
\label{tab:res_expanded_flux}
\end{table}

\subsection{L1-Relative Velocity Change as a Surrogate for Trajectory Variation}
\label{sec:L1_as_proxy}

We use the L1-relative velocity change in Eq.~\ref{eq:velocity_change} as a lightweight, model-internal measure of variation along the generated flow trajectory. Larger values of $\hat v(x,\alpha)$ indicate stronger changes in the learned velocity field across integration timesteps and may therefore reflect greater sensitivity to coarse temporal discretization. Importantly, $\hat v(x,\alpha)$ is not a direct estimate of discretization error or perceptual quality, but an observable trajectory-level surrogate that requires neither a reference solution nor an external quality model.

Figure~\ref{fig:reward_justification} illustrates the velocity-change profiles across different prompts and integration budgets. To more directly examine the relationship with discretization sensitivity, we additionally compare generations obtained with step budgets $\{5,10,\ldots,50\}$ against a 100-step reference generated from the identical initial latent. We measure reference deviation using the $\ell_2$ distance between terminal latent states. Across eight representative GenEval prompts, $\hat v$ is strongly associated with reference deviation as the integration budget varies, with a mean within-prompt Spearman correlation of $0.924$ (median $0.970$); pooled across all prompt--budget pairs, the correlation is $\rho=0.822$ ($p<10^{-16}$). A similar association is observed using the perceptual LPIPS distance (Spearman $\rho=0.676$, $p<10^{-11}$), indicating that the trend is not specific to latent-space distance.

Both $\hat v$ and reference deviation vary systematically with the integration budget, so these correlations should be interpreted as evidence that $\hat v$ tracks \emph{within-trajectory discretization sensitivity}, rather than as evidence that it independently ranks generation difficulty across prompts. Indeed, when controlling for the step budget and comparing prompts at fixed $\alpha$, the mean cross-prompt Spearman correlation is $0.255$ in this pilot. Thus, our analysis supports $\hat v$ as an intrinsic indicator of trajectory variation while distinguishing this role from cross-prompt complexity estimation.

\begin{table}[t]
\centering
\begin{tabular}{lc}
\toprule
Analysis & Spearman $\rho$ \\
\midrule
Pooled prompt--budget pairs & 0.822 \\
Mean within-prompt & 0.924 \\
Median within-prompt & 0.970 \\
Mean cross-prompt at fixed budget & 0.255 \\
\bottomrule
\end{tabular}
\caption{Association between L1-relative velocity variation and deviation from a 100-step reference. All step budgets for each prompt use the identical initial latent.}
\label{tab:velocity_ref_error}
\end{table}

\subsection{Context Matters}
\label{sec:context_matters}

Table~\ref{tab:context_importance} evaluates the role of input context in adaptive step selection. We compare the contextual UCB policy with a non-contextual UCB variant that uses the same bandit framework without conditioning decisions on prompt-specific features. We additionally evaluate Thompson Sampling (TS) as an alternative exploration strategy.

The contextual variants achieve a better quality--efficiency trade-off than the non-contextual UCB baseline. In particular, contextual UCB improves the overall GenEval score from $0.74$ to $0.77$ while reducing the average NFE from $20.2$ to $17.7$. Thompson Sampling exhibits similar behavior, reaching an overall score of $0.76$ at $17.5$ NFE. These results indicate that conditioning the decision policy on input-dependent information is useful beyond learning a single dataset-level step-selection strategy.

\begin{table}[t]
\centering
\small
\setlength{\tabcolsep}{6pt}
\renewcommand{\arraystretch}{1.15}
\begin{tabular}{lcccc}
\toprule
\textbf{Method} & \textbf{Overall $\uparrow$} & \textbf{IQA $\uparrow$} & \textbf{NFE $\downarrow$} & \textbf{Time $\downarrow$} \\
\midrule
Ours + UCB (No Context) & 0.74 & 0.60 & 20.2 & 31.9 \\
\midrule
\textbf{Ours + UCB} & \textbf{0.77} & \textbf{0.65} & 17.7 & 27.5 \\
Ours + TS & 0.76 & \textbf{0.65} & \textbf{17.5} & \textbf{26.9} \\
\bottomrule
\end{tabular}
\caption{\textbf{Ablation of context and bandit strategy.} Conditioning the policy on input context improves the quality--efficiency trade-off relative to the non-contextual UCB variant, while UCB and Thompson Sampling provide comparable contextual policies.}
\label{tab:context_importance}
\end{table}

\paragraph{Effect of Context Representation.}
We further ablate the representation used to condition the step-selection policy by comparing no input context, handcrafted prompt features, MiniLM embeddings, and their combination. Table~\ref{tab:context_representation} reports results aggregated over three seeds. The full representation reduces the mean NFE from $39.75$ to $37.67$ and decreases the frequency of selecting the maximum step budget from $65.4\%$ to $47.9\%$ relative to the no-context variant. Handcrafted features achieve a similar reduction in NFE ($37.79$), whereas MiniLM alone remains closer to the no-context policy ($39.58$ NFE). These results indicate that contextual information influences computation allocation, with handcrafted features accounting for a substantial portion of the observed reduction.

\begin{table}[t]
\centering
\small
\setlength{\tabcolsep}{6pt}
\renewcommand{\arraystretch}{1.15}
\begin{tabular}{lcc}
\toprule
\textbf{Context} & \textbf{Mean NFE $\downarrow$} & \textbf{Max-step rate $\downarrow$} \\
\midrule
No context       & 39.75 & 65.4\% \\
Handcrafted only & 37.79 & 52.9\% \\
MiniLM only      & 39.58 & 60.0\% \\
Full             & \textbf{37.67} & \textbf{47.9\%} \\
\bottomrule
\end{tabular}
\caption{\textbf{Effect of context representation.} Results are aggregated over three seeds. Max-step rate denotes the fraction of decisions selecting the largest available step budget.}
\label{tab:context_representation}
\end{table}

\paragraph{Online Learning and Order Robustness.}
Since \algo{} updates its step-selection policy online, we examine its behavior as a function of the number of observed interactions and its sensitivity to prompt ordering. We evaluate 100 GenEval prompts under five random prompt permutations and three independent seeds per permutation. As shown in Table~\ref{tab:online_learning}, the policy improves during the initial interactions and subsequently stabilizes: the mean NFE decreases from $43.2$ over rounds $1$--$10$ to $37.4$ over rounds $81$--$100$, while the mean surrogate reward increases from $0.520$ to $0.580$. Across prompt permutations, the final mean NFE is $37.8\pm1.1$ and the corresponding mean surrogate reward is $0.579\pm0.009$, indicating limited sensitivity to prompt ordering.

\begin{table}[t]
\centering
\small
\setlength{\tabcolsep}{6pt}
\renewcommand{\arraystretch}{1.12}
\begin{tabular}{lccc}
\toprule
\textbf{Rounds} &
\textbf{Reward $\uparrow$} &
\textbf{NFE $\downarrow$} &
\textbf{Max-step rate} \\
\midrule
1--10   & 0.520 & 43.2 & 72.0\% \\
11--20  & 0.548 & 40.5 & 63.3\% \\
21--40  & 0.566 & 38.7 & 55.8\% \\
41--60  & 0.575 & 37.9 & 51.6\% \\
61--80  & 0.578 & 37.6 & 50.4\% \\
81--100 & 0.580 & 37.4 & 49.8\% \\
\bottomrule
\end{tabular}
\caption{\textbf{Online adaptation and order robustness.} Results are aggregated over five random prompt permutations and three independent seeds per permutation. Max-step rate denotes the fraction of decisions selecting the largest available step budget.}
\label{tab:online_learning}
\end{table}

\paragraph{Effect of UCB Exploration.}
To isolate the contribution of uncertainty-aware exploration, we compare NeuralUCB with a greedy variant using the identical context representation, neural reward model, online updates, replay buffer, initialization, and cold-start procedure, but without the UCB confidence bonus after initialization. The greedy policy therefore selects $\alpha_t=\arg\max_{\alpha}f_\theta(\phi(x_t),\alpha)$. As shown in Table~\ref{tab:ucb_greedy}, NeuralUCB achieves a higher mean surrogate reward ($0.582$ vs.\ $0.541$) while reducing the mean NFE ($37.5$ vs.\ $39.6$). It also selects the maximum step budget less frequently ($49.7\%$ vs.\ $63.4\%$) and exhibits higher action entropy ($1.54$ vs.\ $1.21$), indicating a more diverse allocation of computation. This controlled ablation shows that uncertainty-aware exploration provides a measurable benefit beyond greedy optimization of the learned reward predictor.

\begin{table}[t]
\centering
\small
\setlength{\tabcolsep}{6pt}
\renewcommand{\arraystretch}{1.12}
\begin{tabular}{lcccc}
\toprule
\textbf{Policy} &
\textbf{Reward $\uparrow$} &
\textbf{NFE $\downarrow$} &
\textbf{Max-step rate} &
\textbf{Action entropy $\uparrow$} \\
\midrule
Greedy &
$0.541\pm0.011$ &
$39.6\pm1.2$ &
$63.4\%$ &
$1.21$ \\
NeuralUCB &
$\mathbf{0.582\pm0.008}$ &
$\mathbf{37.5\pm0.8}$ &
$49.7\%$ &
$\mathbf{1.54}$ \\
\bottomrule
\end{tabular}
\caption{\textbf{Effect of UCB exploration.} NeuralUCB is compared with an otherwise identical online greedy policy with the uncertainty bonus removed after initialization. We report mean $\pm$ standard deviation over three independent seeds. Action entropy denotes the entropy of the empirical distribution of selected step budgets.}
\label{tab:ucb_greedy}
\end{table}

\begin{table}[t]
\centering
\small
\setlength{\tabcolsep}{4pt}
\renewcommand{\arraystretch}{1.2}

\begin{tabular}{l|cccccc|ccc}
\toprule
\textbf{Method} 
& \multicolumn{6}{c|}{\textbf{Compositional Dimensions}} 
& \multicolumn{3}{c}{\textbf{Overall Metrics}} \\

\cmidrule(lr){2-7} \cmidrule(lr){8-10}

& \textbf{SO} & \textbf{TO} & \textbf{CO} & \textbf{CL} & \textbf{CA} & \textbf{PO}
& \textbf{Overall} & \textbf{CLIPIQA} & \textbf{NFE} \\

\midrule
\multicolumn{10}{c}{\textit{Full Model}} \\
\midrule
Full-10 & 0.99 & 0.79 & 0.37 & 0.88 & 0.21 & 0.15 & 0.57 & 0.72 & 10 \\
Full-6  & 0.99 & 0.78 & 0.36 & 0.86 & 0.20 & 0.15 & 0.55 & 0.68 & 6 \\
Full-3  & 0.98 & 0.75 & 0.29 & 0.85 & 0.18 & 0.12 & 0.52 & 0.62 & 3 \\

\midrule
\multicolumn{10}{c}{\textit{Ours}} \\
\midrule
\textbf{Ours-UCB} 
& \textbf{0.99} & \textbf{0.79} & 0.35 & \textbf{0.87} & 0.20 & 0.13
& \textbf{0.55} & \textbf{0.70} & \textbf{5.1} \\

\bottomrule
\end{tabular}

\caption{Comparison under low-step regimes where the PeRFlow model, which is trained for low step generation within $5$, $8$, $10$ steps, where our method further improves this method by making the choice of number of steps adaptive. Our method achieves competitive performance while reducing the number of function evaluations (NFE).}
\label{tab:res_perflow}
\end{table}

\begin{figure}
    \centering
    \includegraphics[width=0.5\linewidth]{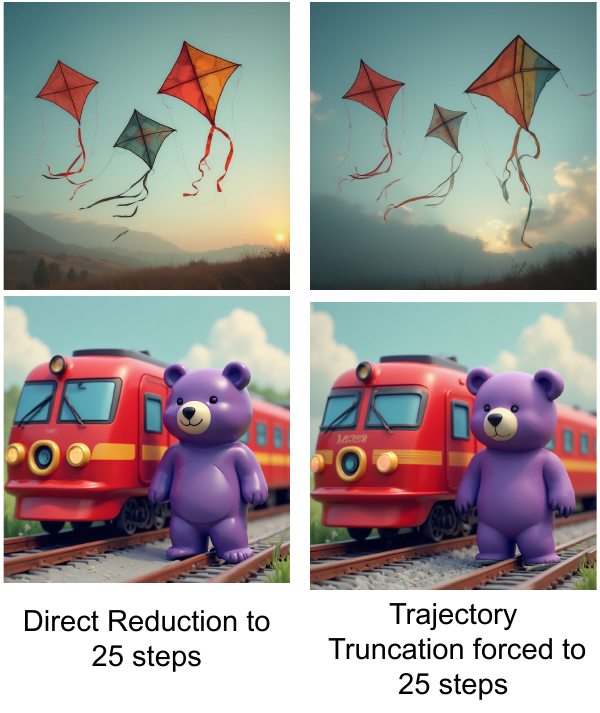}
    \caption{\textbf{Direct step reduction vs.\ trajectory truncation.} Direct reduction generates with a timestep schedule constructed for the reduced step budget, whereas trajectory truncation obtains a comparable computational reduction by skipping intermediate evaluations from a higher-step trajectory.}
    \label{fig:trajectory}
\end{figure}

\begin{table}[t]
\centering
\small
\setlength{\tabcolsep}{5pt}
\renewcommand{\arraystretch}{1.2}
\begin{tabular}{lcccc}
\toprule
\textbf{Method} & \textbf{GenEval $\uparrow$} & \textbf{CLIP-IQA $\uparrow$} & \textbf{NFE $\downarrow$} & \textbf{Time $\downarrow$} \\
\midrule

\multicolumn{5}{c}{\textit{Baseline Methods}} \\
\midrule
TeaCache  & 0.74 & 0.59 & 29.4 & 45.8 \\
FastFlow  & 0.75 & 0.59 & 20.9 & 32.9 \\
FlowCast  & \textbf{0.76} & \textbf{0.60} & 21.5 & 34.2 \\

\midrule
\multicolumn{5}{c}{\textit{With Ours (Adaptive Step Selection)}} \\
\midrule
TeaCache + Ours & 0.71 & 0.55 & 16.8 & 25.7 \\
FastFlow + Ours & 0.73 & 0.59 & \textbf{16.0} & \textbf{24.5} \\
FlowCast + Ours & 0.75 & 0.59 & -- & 26.3 \\

\bottomrule
\end{tabular}
\caption{\textbf{Combination with existing acceleration methods on BAGEL.} Adaptive step-budget selection can be applied together with TeaCache, FastFlow, and FlowCast, providing additional runtime reductions with method-dependent changes in generation quality.}
\label{tab:complementary_results}
\end{table}





\begin{table}[t]
\centering
\small
\setlength{\tabcolsep}{5pt}
\renewcommand{\arraystretch}{1.2}
\begin{tabular}{l|cc|cc}
\toprule
\multirow{2}{*}{\textbf{Hyperparameter}} 
& \multicolumn{2}{c|}{\textbf{FLUX}} 
& \multicolumn{2}{c}{\textbf{BAGEL}} \\
\cmidrule(lr){2-3} \cmidrule(lr){4-5}
& \textbf{Generation} & \textbf{Editing} 
& \textbf{Generation} & \textbf{Editing} \\
\midrule

Image resolution 
& $1360 \times 768$ & -- 
& $1024 \times 1024$ & -- \\

Guidance scale 
& $2.5$ & $2.5$ 
& $4.0$ & $4.0$ \\

Guidance rescale 
& $0$ & $0$ 
& $1.0$ & $2.0$ \\

$\mu$ 
& $\tfrac{1}{50}$ & $\tfrac{1}{60}$ 
& $\tfrac{1}{50}$ & $\tfrac{1}{60}$ \\

CFG interval 
& -- & -- 
& $[0.4,\,1.0]$ & $[0.0,\,1.0]$ \\

Timestep shift 
& -- & -- 
& $3$ & $3$ \\

CFG renorm min 
& -- & -- 
& $0$ & $0$ \\

CFG renorm type 
& -- & -- 
& global & text-channel \\

Arms 
& \multicolumn{2}{c|}{$\{10, 15, 20, \dots, 50\}$} 
& \multicolumn{2}{c}{$\text{Same as FLUX}$} \\

\bottomrule
\end{tabular}
\caption{Hyperparameter settings for FLUX and BAGEL across generation and editing tasks.}
\label{tab:hyperparams}
\end{table}

\subsection{Direct Step Reduction vs.\ Trajectory Manipulation}
\label{sec:direct_vs_trajectory}

Figure~\ref{fig:trajectory} compares direct step reduction with trajectory truncation. In the former, generation is performed using a timestep schedule constructed directly for the reduced integration budget. In the latter, a comparable speedup is obtained by skipping intermediate evaluations along a higher-step trajectory.

In the examples shown, direct step reduction retains better visual quality than trajectory truncation at a comparable computational budget. One possible explanation is that direct reduction preserves a solver-consistent discretization for the selected number of steps, whereas trajectory truncation modifies the original sampling trajectory by omitting intermediate evaluations.

This comparison motivates the design of \algo{}, which selects an input-dependent step budget and then executes the corresponding solver schedule rather than selecting a high-step trajectory and subsequently removing intermediate evaluations.

\begin{table}[t]
\centering
\small
\setlength{\tabcolsep}{5pt}
\renewcommand{\arraystretch}{1.2}
\begin{tabular}{l l c c c}
\toprule
\textbf{Method} & \textbf{Characteristics} & \textbf{GenEval} & \textbf{IQA} & \textbf{NFE} \\
\midrule

\multicolumn{5}{c}{\textit{Base Solvers (50 Steps)}} \\
\midrule
Euler        & Deterministic, 1st-order     & 0.60 & 0.71 & 50 \\
Euler-anc    & Stochastic, 1st-order        & 0.57 & 0.72 & 50 \\
DPM++2M      & Deterministic, multistep     & 0.59 & 0.70 & 50 \\
DPM++2S-anc  & Stochastic, multistep        & 0.57 & 0.71 & 50 \\
Heun         & Deterministic, 2nd-order     & 0.59 & 0.70 & 50 \\
LMS          & Deterministic, multistep     & 0.58 & 0.56 & 50 \\

\midrule
\multicolumn{5}{c}{\textit{CoFlow (Adaptive Steps)}} \\
\midrule
CoFlow-Euler        & Deterministic, 1st-order     & 0.59 & 0.70 & 19.4 \\
CoFlow-Euler-anc    & Stochastic, 1st-order        & 0.56 & 0.68 & 18.9 \\
CoFlow-DPM++2M      & Deterministic, multistep     & 0.57 & 0.68 & 18.5 \\
CoFlow-DPM++2S-anc  & Stochastic, multistep        & 0.55 & 0.70 & 18.3 \\
CoFlow-Heun         & Deterministic, 2nd-order     & 0.59 & 0.68 & 19.8 \\
CoFlow-LMS          & Deterministic, multistep     & 0.57 & 0.54 & 19.1 \\

\bottomrule
\end{tabular}
\caption{Comparison across solver families on a Stable Diffusion backbone. We categorize solvers by stochasticity and numerical integration type. CoFlow consistently reduces the number of function evaluations (NFE) by $\sim$2.5$\times$ while maintaining comparable generation quality.}
\label{tab:solver_comparison}
\end{table}

\begin{table}[t]
\centering
\small
\setlength{\tabcolsep}{4pt}
\renewcommand{\arraystretch}{1.2}
\begin{tabular}{l p{3.8cm} p{5.2cm}}
\toprule
\textbf{Feature} & \textbf{Definition} & \textbf{Description} \\
\midrule

Word count 
& $\min\left(\frac{\#\text{words}}{50},\,1\right)$ 
& Captures prompt length and compositional richness. Longer prompts typically involve more entities, relations, or attributes, increasing reasoning complexity. \\

Subword diversity (TTR) 
& $\frac{|\text{unique tokens}|}{\#\text{tokens}}$ 
& Measures lexical diversity. Higher diversity indicates richer semantics and reduced redundancy, often correlating with more complex generation requirements. \\

Mean word length 
& $\min\left(\frac{1}{10}\cdot\frac{\sum |w|}{\#\text{words}},\,1\right)$ 
& Acts as a proxy for linguistic specificity. Longer words often correspond to more descriptive or specialized concepts. \\

Numeric fraction 
& $\frac{\#\text{words containing digits}}{\#\text{words}}$ 
& Captures numerical constraints (e.g., counts or quantities), which increase compositional difficulty and require precise control in generation. \\

Object density (POS-based)
& $\min\left(\frac{\#\text{nouns}}{\#\text{words}},\,1\right)$ 
& Measures the number of entities in the prompt. Higher density implies more objects and interactions, increasing spatial and compositional complexity. \\

Attribute density (POS-based) 
& $\min\left(\frac{\#\text{adjectives}}{\#\text{words}},\,1\right)$ 
& Captures descriptive richness. More attributes introduce finer-grained constraints (e.g., color, size), making generation more demanding. \\

Punctuation density 
& $\min\left(\frac{\#\text{punctuation}}{2 \cdot \#\text{words}},\,1\right)$ 
& Reflects structural complexity such as clauses or multi-part instructions. Higher density often corresponds to more intricate prompt structure. \\

Syllable complexity 
& $\min\left(\frac{\text{avg syllables per word}}{4},\,1\right)$ 
& Approximates linguistic complexity and descriptiveness. Higher syllable counts correlate with more detailed or specialized language. \\

\bottomrule
\end{tabular}
\caption{Handcrafted prompt features used alongside MiniLM embeddings to construct the context vector $\phi(x)$. The features provide lightweight compositional, lexical, numerical, and structural cues for contextual step selection.} 
\label{tab:prompt_features}
\end{table}

\subsection{Performance Across Solver Families}\label{sec:solver_families}
In Table~\ref{tab:solver_comparison}, we evaluate \algo{} across a diverse set of numerical solvers on the Stable Diffusion XL \cite{podell2023sdxl} backbone, including deterministic and stochastic variants, as well as first-order, second-order, and multistep integration schemes. We first report baseline performance under a fixed discretization budget (50 steps) to establish a consistent comparison across solvers. We then apply \algo{} as an adaptive inference controller that selects a reduced, input-dependent number of steps while preserving the original solver dynamics.

Across all solver families, \algo{} achieves a consistent reduction in the number of function evaluations (NFEs) by approximately $2.5\times$, while maintaining competitive performance in both semantic alignment (GenEval) and perceptual quality (IQA). Notably, the relative ranking of solvers remains largely unchanged after applying \algo{}, indicating that our method operates orthogonally to the choice of numerical solver. 

We also observe mild solver-dependent sensitivity to step reduction: multistep methods (e.g., DPM++ variants) exhibit slightly larger degradation in GenEval, likely due to their reliance on historical trajectory information, whereas higher-order solvers such as Heun demonstrate greater robustness under reduced discretization. Overall, these results highlight that \algo{} provides a general, plug-and-play acceleration mechanism that is compatible with a wide range of solvers, enabling substantial speedups without sacrificing generation fidelity.

\section{Context Features}\label{sec:context_features}
We construct a context vector $\phi(x)\in\mathbb{R}^d$ from complementary semantic, compositional, and structural properties of the input prompt. The representation combines a MiniLM sentence embedding with lightweight handcrafted features, providing input-dependent information to the contextual bandit while introducing minimal computational overhead. The individual handcrafted features are summarized in Table~\ref{tab:prompt_features}.


\textbf{Feature Design:}
The handcrafted features capture complementary surface-level properties of the prompt that may provide useful cues for adaptive step selection:

\begin{itemize}
    \item \textbf{Compositional scale:} Word count captures prompt length and provides a lightweight cue for the amount of compositional information present.

    \item \textbf{Semantic richness:} Subword type-token ratio (TTR), mean word length, and syllable complexity capture lexical diversity and linguistic specificity.

    \item \textbf{Entities and attributes:} Object and attribute densities are obtained using lightweight part-of-speech tagging and provide cues about the presence of nouns and descriptive constraints.

    \item \textbf{Numerical constraints:} The fraction of words containing digits captures explicit numerical or counting information.

    \item \textbf{Structural information:} Punctuation density provides a lightweight description of prompt structure, including clauses, lists, and multi-part instructions.
\end{itemize}

\textbf{Normalization:}
All features are normalized to the range $[0,1]$ to ensure consistent scaling across different dimensions. This prevents any single feature from dominating the representation and stabilizes learning in the bandit model.

\textbf{Efficiency:}
The proposed feature set is intentionally lightweight. All statistics are computed in linear time with respect to the prompt length, and POS-based features are extracted using a small, off-the-shelf tagger. In practice, this incurs only a few milliseconds of overhead per prompt, which is negligible compared to the cost of generative inference.

\subsection{Choice of Context Representation as MiniLM}\label{sec:choice_of_minilm}

A key component of our approach is the construction of a lightweight yet expressive context representation for each prompt. We employ a pretrained SentenceTransformer (all-MiniLM-L6-v2) to obtain semantic embeddings, which are combined with handcrafted complexity features.

\textbf{Why SentenceTransformer (MiniLM)?}
We use all-MiniLM-L6-v2 as a compact semantic representation of the prompt. Its sentence-level embeddings complement the handcrafted lexical and structural features by encoding higher-level semantic information while remaining inexpensive to compute. Its compact representation is particularly suitable for the lightweight contextual decision module used by \algo{}.

\textbf{Why not use the model’s native text encoder?}
An alternative is to use the prompt encoder from the underlying generative model. However, we avoid this as it produces high-dimensional representations, leading to significantly larger input sizes for the bandit. This increases both computational cost and training instability, making it less suitable for a lightweight, real-time decision module.
In contrast, our use of an external, pretrained sentence encoder provides a \emph{model-agnostic} and reusable representation that generalizes across different generative backbones (e.g., BAGEL, FLUX) without modification. Overall, this design yields a compact, semantically meaningful, and model-agnostic context representation that is suited for adaptive decision-making under strict efficiency constraints.\\


    \FloatBarrier
    \clearpage
    \section*{NeurIPS Paper Checklist}

\begin{enumerate}

\item {\bf Claims}
    \item[] Question: Do the main claims made in the abstract and introduction accurately reflect the paper's contributions and scope?
    \item[] Answer: \answerYes{} 
    \item[] Justification: Yes, the main claims in the abstract and introduction accurately reflect the paper's contribution and scope.

\item {\bf Limitations}
    \item[] Question: Does the paper discuss the limitations of the work performed by the authors?
    \item[] Answer: \answerYes{} 
    \item[] Justification: Added the limitation of our work in \ref{sec:limitations}

\item {\bf Theory assumptions and proofs}
    \item[] Question: For each theoretical result, does the paper provide the full set of assumptions and a complete (and correct) proof?
    \item[] Answer: \answerYes{} 
    \item[] Justification: In the Appendix \ref{sec:proof}, we add the proof of Theorem \ref{thm:main}.

    \item {\bf Experimental result reproducibility}
    \item[] Question: Does the paper fully disclose all the information needed to reproduce the main experimental results of the paper to the extent that it affects the main claims and/or conclusions of the paper (regardless of whether the code and data are provided or not)?
    \item[] Answer: \answerYes{} 
    \item[] Justification: We provided the source code in the abstract, all the hyperparameter details in \ref{tab:hyperparams} to reproduce the results.

\item {\bf Open access to data and code}
    \item[] Question: Does the paper provide open access to the data and code, with sufficient instructions to faithfully reproduce the main experimental results, as described in supplemental material?
    \item[] Answer: \answerYes{} 
    \item[] Justification: We provide an anonymized repository for the source code to reproduce the experimental results. The link of the repository is provided in the abstract of the paper.

\item {\bf Experimental setting/details}
    \item[] Question: Does the paper specify all the training and test details (e.g., data splits, hyperparameters, how they were chosen, type of optimizer) necessary to understand the results?
    \item[] Answer: \answerYes{} 
    \item[] Justification: Our method is fully inference time, so no train splits are present, and we provide all hyperparameter settings in Table \ref{tab:hyperparams} and in the experimental section \ref{sec:experiments}

\item {\bf Experiment statistical significance}
    \item[] Question: Does the paper report error bars suitably and correctly defined or other appropriate information about the statistical significance of the experiments?
    \item[] Answer: \answerNo{} 
    \item[] Justification: Consistent with prior work in visual generation, evaluation is primarily based on standard benchmark metrics and qualitative comparisons.

\item {\bf Experiments compute resources}
    \item[] Question: For each experiment, does the paper provide sufficient information on the computer resources (type of compute workers, memory, time of execution) needed to reproduce the experiments?
    \item[] Answer: \answerYes{} 
    \item[] Justification: In the experimental section \ref{sec:experiments}, we provide the compute resources used and the time for running experiments is discussed in all the result tables, as our method works on inference time reduction.
    
\item {\bf Code of ethics}
    \item[] Question: Does the research conducted in the paper conform, in every respect, with the NeurIPS Code of Ethics \url{https://neurips.cc/public/EthicsGuidelines}?
    \item[] Answer: \answerYes{} 
    \item[] Justification: Our paper conforms, in every aspect, with NeurIPS Code of Ethics. 

\item {\bf Broader impacts}
    \item[] Question: Does the paper discuss both potential positive societal impacts and negative societal impacts of the work performed?
    \item[] Answer: \answerNA{} 
    \item[] Justification: Our work is focused on inference time speedup of existing methods, hence it does not have any societal impact. 
    
\item {\bf Safeguards}
    \item[] Question: Does the paper describe safeguards that have been put in place for responsible release of data or models that have a high risk for misuse (e.g., pre-trained language models, image generators, or scraped datasets)?
    \item[] Answer: \answerNA{} 
    \item[] Justification: As our method is an inference-time speedup method, we do not require any additional safeguards to be put.

\item {\bf Licenses for existing assets}
    \item[] Question: Are the creators or original owners of assets (e.g., code, data, models), used in the paper, properly credited and are the license and terms of use explicitly mentioned and properly respected?
    \item[] Answer: \answerYes{} 
    \item[] Justification: We use some of the existing models in the visual generation domain, and we have properly cited them and adhered to their terms of use.

\item {\bf New assets}
    \item[] Question: Are new assets introduced in the paper well documented and is the documentation provided alongside the assets?
    \item[] Answer: \answerNA{} 
    \item[] Justification: We do not provide any new assets.

\item {\bf Crowdsourcing and research with human subjects}
    \item[] Question: For crowdsourcing experiments and research with human subjects, does the paper include the full text of instructions given to participants and screenshots, if applicable, as well as details about compensation (if any)? 
    \item[] Answer: \answerNA{} 
    \item[] Justification: There are no crowdsourcing experiments in our work.

\item {\bf Institutional review board (IRB) approvals or equivalent for research with human subjects}
    \item[] Question: Does the paper describe potential risks incurred by study participants, whether such risks were disclosed to the subjects, and whether Institutional Review Board (IRB) approvals (or an equivalent approval/review based on the requirements of your country or institution) were obtained?
    \item[] Answer: \answerNA{} 
    \item[] Justification: There are no potential risks of our work. Our work is focused on inference time speedup.

\item {\bf Declaration of LLM usage}
    \item[] Question: Does the paper describe the usage of LLMs if it is an important, original, or non-standard component of the core methods in this research? Note that if the LLM is used only for writing, editing, or formatting purposes and does \emph{not} impact the core methodology, scientific rigor, or originality of the research, declaration is not required.
    \item[] Answer: \answerNA{} 
    \item[] Justification: We use the LLM only for editing and formatting.

\end{enumerate}

\end{document}